\documentclass[11pt]{article}

\usepackage{times}
\usepackage{graphicx} %
\usepackage{algorithm}
\usepackage{algorithmic}
\usepackage{natbib}
\usepackage{mathrsfs}
\usepackage{latexsym}
\usepackage{amsmath}
\usepackage{amsfonts}
\usepackage{url}
\usepackage{amssymb}
\usepackage{amsthm}
\usepackage{tikz}
\usetikzlibrary{positioning}
\usepackage{epstopdf}
\usepackage{hyperref}

\usepackage[T1]{fontenc}
\usepackage[utf8]{inputenc}
\usepackage{authblk}

\usepackage{graphicx} %
\usepackage{subfigure} 
\usepackage{epstopdf}
\usepackage{float}

\usepackage{natbib}

\usepackage{stmaryrd}  %
\newcommand{\pred}[1]{\left\llbracket #1 \right\rrbracket}

\newcommand{\R}{\mathbb{R}}

\newcommand{\N}{\mathbb{N}}

\newcommand{\paren}[1]{\left( #1 \right)}

\newcommand{\beq}{\begin{eqnarray*}}
\newcommand{\eeq}{\end{eqnarray*}}
\newcommand{\beqn}{\begin{eqnarray}}
\newcommand{\eeqn}{\end{eqnarray}}
\newcommand{\ben}{\begin{enumerate}}
\newcommand{\een}{\end{enumerate}}
\newcommand{\bit}{\begin{itemize}}
\newcommand{\eit}{\end{itemize}}
\newcommand{\hide}[1]{}
\newcommand{\oo}[1]{\frac{1}{#1}}
\newcommand{\eps}{\epsilon}

\newcommand{\argmax}{\mathop{\mathrm{argmax}}}

\newcommand{\inv}{^{-1}} %

\newcommand{\chr}{\boldsymbol{\mathbbm{1}}} %

\newcommand{\rad}{\operatorname{rad}}

\newcommand{\gn}{\, | \,}

\newcommand{\calM}{\mathcal{M}}
\newcommand{\M}{\calM}
\newcommand{\calN}{\mathcal{N}}

\newcommand{\calX}{\mathcal{X}}

\newcommand{\X}{\calX}
\newcommand{\Y}{\mathcal{Y}}

\newcommand{\F}{\mathcal{F}}

\newcommand{\marg}{\operatorname{marg}}
\newcommand{\dist}{\rho} %

\newcommand{\ds}{\displaystyle}

\newcommand{\bepf}{\begin{proof}}
\newcommand{\enpf}{\end{proof}}

\newcommand{\sign}{\operatorname{sign}}
\newcommand{\ddim}{\operatorname{ddim}}
\newcommand{\dens}{\operatorname{dens}}

\newcommand{\err}{\operatorname{err}}
\newcommand{\serr}{\widehat{\operatorname{err}}}
\newcommand{\tips}{\tilde\eps}
\newcommand{\bin}{\operatorname{Bin}}
\renewcommand{\P}{\mathbb{P}}
\def\eps{\varepsilon}

\usepackage{bbm}
\renewcommand{\chr}{\boldsymbol{\mathbbm{1}}} %
\renewcommand{\pred}[1]{\chr_{\left\{ #1 \right\}}}

\newcommand{\set}[1]{\left\{ #1 \right\}}
\newcommand{\tsset}[1]{\{ #1 \}}

\newtheorem{theorem}{Theorem}

\newtheorem{lemma}[theorem]{Lemma}

\usepackage{enumerate}%

\usepackage{amsmath,amssymb}
\usepackage{algorithm}
\usepackage{algorithmic}

\usepackage{tikz}
\usetikzlibrary{positioning}

\usepackage{geometry}
\title{
Nearly optimal classification for semimetrics
}
\author[*]{Lee-Ad Gottlieb}
\author[**]{Aryeh Kontorovich}
\affil[*]{Department of Computer Science, Ariel University}
\affil[**]{Department of Computer Science, Ben-Gurion University}

\begin{document}
\maketitle

\begin{abstract}
We initiate the rigorous study of classification in semimetric spaces,
which are
point sets with a distance function 
that is non-negative and symmetric, but
need not satisfy the triangle inequality.
For metric spaces,
%it is known that in metric spaces, 
the doubling dimension essentially characterizes
both the runtime and sample complexity of classification algorithms
---
yet we show that this is not the case for semimetrics. 
Instead, we define the {\em density dimension} and
discover that it plays a central role in
%establish its role in characterizing
the statistical and algorithmic feasibility of learning in semimetric spaces.
We present nearly optimal sample compression algorithms
and use these to obtain generalization guarantees, including fast rates.
The latter hold for general sample compression schemes and may be of independent interest.
\end{abstract}

\section{Introduction}
\label{sec:intro}
The problem of learning in non-metric spaces has been of significant recent interest,
being the subject of a 2010 COLT workshop and a central topic of all three SIMBAD 
conferences. In this paper, we 
initiate the study of efficient statistical learning in
{\em semimetric} spaces, which are
point sets endowed with a distance function that 
is non-negative and symmetric but may not satisfy
the triangle inequality \citep{wilson-semi-metric}\footnote{
Some authors use the term {``semimetric''} to mean
{\em pseudometrics}.
These
preserve much of the structure of metrics,
the only difference being
that they allow distinct points to have distance $0$.
Our usage appears to be the standard one.
}.
Without the latter, quite a bit of structure is lost
--- for example,
semimetric spaces admit convergent sequences without a Cauchy subsequence \citep{burke1972}.
We are not aware of any rigorous learning results in semimetric spaces prior to this work.

\paragraph{Background and motivation.}
Much of the existing machinery for classification algorithms, as well as
generalization bounds, depends strongly on the data residing in a Hilbert space. 
For some important applications, this structural
constraint severely limits the applicability of existing methods.
Indeed, it is often the case that the data is naturally endowed with some
metric strongly dissimilar to the familiar Euclidean norm.

Consider images, for example.
Although these can be naively represented as coordinate-vectors in
$\R^d$, the Euclidean (or even $\ell_p$) distance between the representative
vectors does not correspond well to the one perceived
by human vision. Instead, the earthmover distance
is commonly used in vision applications
\citep{DBLP:journals/ijcv/RubnerTG00}.
Yet representing earthmover distances using any fixed $\ell_p$ norm
unavoidably introduces very large inter-point distortion \citep{NS07},
potentially corrupting the data geometry before the learning process has even begun.
Nor is this issue mitigated by kernelization, as kernels
necessarily embed the data in a Hilbert space, again incurring the aforementioned distortion. 
A similar issue arises
for strings: These can be naively treated as
vectors endowed with different $\ell_p$ metrics,
but a much more natural metric over strings is the edit distance, which is
similarly known to be strongly non-Euclidean \citep{AK10}.
Additional limitations of kernel methods are articulated in \citet{DBLP:journals/ml/BalcanBS08}.

These concerns have led researchers to seek out algorithmic and statistical
approaches that apply in greater generality.
A particularly fruitful recent direction has focused on metric spaces.
Metric spaces are point sets endowed with a distance function that is
non-negative and symmetric, and also satisfies the triangle equality.
Since metric spaces may be highly complex --- for example, they 
include infinite-dimensional Hilbert spaces ---
%Hence,
the discussion is typically restricted to 
%particular attention has been paid to 
metric spaces with bounded 
{\em intrinsic dimension}.
The latter
%which 
may be
formalized, e.g., via metric entropy numbers or
the
%, such as small covering-number growth rate 
%or 
doubling dimension.
This paradigm captures some natural distance metrics, such as earthmover
and edit distances \citep{DBLP:journals/tit/GottliebKK14}.

\begin{sloppypar}
%Striving for maximal generality, no additional 
Assuming no additional structure beyond inter-point distances,
one is left (almost tautologically) with proximity-based methods
---
and
all the learning algorithms considered in this paper will be
%and in this paper our learning algorithms are 
%``learning algorithms'' will henceforth refer to 
variants of the Nearest Neighbor classifier.
For metric spaces,
it is known that a sample of size exponential in the doubling dimension 
($\ddim$) suffices to
achieve low generalization error \citep{DBLP:journals/jmlr/LuxburgB04,DBLP:conf/colt/GottliebKK10,shwartz2014understanding,DBLP:conf/icml/KontorovichW14},
and that exponential dependence on $\ddim$ is in general unavoidable \citep{shwartz2014understanding}.
As for algorithmic runtimes, the naive nearest-neighbor classifier
evaluates queries in $O(n)$ time (where $n$ is the sample size);
however, an approximate nearest neighbor can be found in time $2^{O(\ddim)} \log n$.
If one desires runtimes depending not on $n$ but on the geometry (say, margin $\gamma$)
of the data, one may achieve a sample compression scheme of size $\gamma^{-O(\ddim)}$,
and it is NP-hard to achieve a significantly better compression \citep{DBLP:conf/nips/GottliebKN14}.
Hence, the doubling dimension in some sense characterizes the statistical 
and computational difficulty of learning in metric spaces.
We note that all learning bounds and algorithms for doubling spaces 
rely on
the packing property 
for 
%doubling metrics 
these
spaces (Lemma~\ref{lem:metric-pack}), which 
upper-bounds the size of
a point set 
whose inter-point distance is bounded from below.
%with minimum 
\end{sloppypar}

\begin{sloppypar}
While metric spaces are significantly more general than Hilbertian ones, 
they still do not capture many common distance functions used by practitioners.
These non-metric distances include
the Jensen-Shannon divergence, which appears in statistical applications
\citep{fuglede04,DBLP:conf/nips/GoodfellowPMXWOCB14},
$k$-median Hausdorff distances and $\ell_p$ distances with $0<p<1$,
which appear in vision applications \citep{jain94,DBLP:journals/pami/JacobsWG00}
--- all of which are
%These distances are readily recognized as 
semimetrics. 
An additional line of work by
\citet{jain94}
and
\citet{DBLP:journals/pami/JacobsWG00,DBLP:conf/iccv/JacobsWG98,DBLP:conf/nips/WeinshallJG98}
%established 
%championed
underscored
the 
%importance and 
effectiveness of non-metric distances in various applications (mainly vision),
and among these, semimetrics again play a prominent role
\citep{514678,546971,817410,232073,Jain:1997:RRH:273392.273401,790412}.
%We are not aware of any rigorous results for learning in semimetric spaces.
\end{sloppypar}

\paragraph{Main results.}
We initiate the rigorous study of classification for semimetric spaces.

\begin{sloppypar}
Our first contribution is a fundamental insight into semimetric spaces.
Unlike in metric spaces, where the covering numbers $\calN(\cdot)$ and the packing numbers $\calM(\cdot)$
are related via $\calM(2\eps)\le\calN(\eps)\le\calM(\eps)$
(see e.g., \citet{alon97scalesensitive}),
%While in metric spaces, the covering properties of the doubling constant imply 
%related packing properties, 
violating the triangle inequality breaks this
connection between covering and packing. 
Particularly, for semimetrics, a 
doubling constant (while well-defined) does not imply a packing property
(Lemma~\ref{lem:semi-pack}). As a consequence, the bounds in the host of results constituting
the theory of learning in doubling metric spaces are not applicable to semimetrics.
Crucially, however, we show that semimetrics with a finite {\em density constant} do
obey a packing property (Lemma~\ref{lem:semi-pack}), 
and so the 
latter serves as a natural basis for
%density constant can be used to derive 
statistical and algorithmic bounds for classification in 
these spaces.
%semimetrics.
This insight is developed further in Lemma~\ref{lem:metric-semimetric}: 
%We also provide a deeper insight: 
While for metric spaces the doubling and density
constants are never very far apart, in semimetric spaces the gap may be arbitrarily large.
%they may be quite different 
%(Lemma~\ref{lem:metric-semimetric}).
\end{sloppypar}

However, the above
discussion does
%results do
not imply that
learning results for metric spaces are automatically portable
into semimetrics simply by replacing
the doubling constant by the density constant.
%the density constant for semimetrics 
%serves as a simple replacement for the doubling constant in metrics.
For example, 
although
the nearest-neighbor classifier is still well-defined in semimetric spaces,
and may naively be evaluated on queries in $O(n)$ time,
%yet 
relaxing to approximate nearest neighbors no longer provides the exponential
speedup that it does in metric spaces (Lemma~\ref{lem:semi-nns}).
Simply put, without the triangle inequality, the hierarchy-based search methods, 
such as \citet{BKL06,DBLP:conf/colt/GottliebKK10} and related approaches,
%do not work.
all break down.

%What does work, 
%What continues to work,
Fortunately, there is a technique that
%What 
survives violations of the triangle inequality --- namely,
sample compression.
%however, is sample compression, 
The latter is
achieved by extracting
a $\gamma$-net,
where $\gamma$ is the sample margin (Theorem~\ref{thm:consistent}).
This can be done in runtime
$
\min\set{n^2,n\paren{1/\gamma}^{O(\dens)}}
$,
%--
%worse than is known for doubling metrics -- 
where $\dens$ is the density dimension defined in (\ref{eq:mudef});
%which
%This 
this
is worse than the corresponding state of the art for metric spaces (Lemma~\ref{lem:semi-net}).
The net-extraction procedure in effect compresses the sample from size
$n$ to $(1/\gamma)^{O(\dens)}$,
which is nearly optimal
unless P=NP
%Further, unless P=NP, the runtime complexity cannot be significantly improved 
(Theorem~\ref{thm:hard}).

On the statistical front, we give a compression-based
generalization bound
%rate 
that smoothly interpolates between the consistent $\tilde O(1/n)$
and agnostic $\tilde O(1/\sqrt n)$ decay regimes (Theorem~\ref{thm:gen-fast}).
This ``fast rate'' holds for general compression schemes and may be of independent interest.
Applied to margin-based semimetric sample-compression schemes, 
it yields the 
bound in Theorem~\ref{thm:gen-sep},
which is amenable to efficient 
Structural Risk Minimization
(Theorem \ref{thm:main-alg})
and
%efficiently computable and optimizable
%which 
cannot be 
%significantly
substantially
improved unless P=NP (Theorem~\ref{thm:hard}).
%In Theorem \ref{thm:main-alg}, we give an algorithm that performs
%structural risk minimization using these bounds.
The lower bound in Theorem~\ref{thm:sample-lb}
shows that 
even under margin assumptions, there exist adversarial distributions
forcing
the sample complexity to
be exponential in $\dens$.

\paragraph{Related work.}
\begin{sloppypar}
%A line of work by 
%\citet{jain94}
%and
%\citet{DBLP:journals/pami/JacobsWG00,DBLP:conf/iccv/JacobsWG98,DBLP:conf/nips/WeinshallJG98}
%established the importance and effectiveness of non-metric distances in various applications (mainly vision).
%Among these, semimetrics play a promiment role
%\citep{514678,546971,817410,232073,Jain:1997:RRH:273392.273401,790412}
In a %recent approach,
series of papers,
\citet{DBLP:conf/icml/BalcanB06,DBLP:conf/stoc/BalcanBV08,DBLP:conf/colt/BalcanBS08,DBLP:journals/ml/BalcanBS08}
developed a theory of learning with similarity functions, which resemble kernels but relax the requirement of
being positive definite. Learning is accomplished by embedding the data into an appropriate Euclidean space
and performing large-margin separation. Hence, this approach effectively extracts the implicit Euclidean structure
encoded in the similarity function, but does not seem well-suited for inherently non-Euclidean data.
\citet{DBLP:conf/icml/WangYF07} extended this framework to dissimilarity functions, obtaining analogous results.
%This line of work is mainly information-theoretic as opposed to algorithmic.
\end{sloppypar}

\section{Preliminaries}
\label{sec:prelim}
\paragraph{Semimetric spaces.} 
Throughout this paper, our instance space $\X$ will be endowed with a semimetric $\dist:\X\times\X\to[0,\infty)$,
which is a non-negative symmetric function verifying $\dist(x,x')=0\iff x=x'$ for all $x,x'\in\X$.
If the semimetric space $(\X,\dist)$
additionally satisfies the triangle inequality, $\dist(x,x')\le\dist(x,x'')+\dist(x'',x')$
for all $x,x',x''\in\X$, then $\dist$ is a {\em metric}. 
The distance between two sets 
$A,B$ 
in a semimetric space
is defined by 
${\ds \dist(A,B)=\inf_{x\in A,x'\in B}\dist(x,x')}$.
For $x\in\X$ and $r>0$, denote by 
$B_r(x)=\set{y\in\X:\dist(x,y)< r}$
the open $r$-ball about $x$.
The {\em radius} of a set is the radius of the smallest ball containing it:
$\rad(A)=\inf\set{r>0: \exists x\in A, A\subseteq B_r(x)}$.
%and ${\ds \diam(A):=\sup_{x,x'\in A}\dist(x,x')}$.

\paragraph{Doubling and density constants.} 
Let $\lambda = \lambda(\X)$ be the smallest number such that every
open ball in $\X$ can be covered by $\lambda$ open balls of half the 
radius, where all balls are centered at points of $\X$. Formally,
$$
\lambda(\X) = 
{
\min\tsset{\lambda\in\N:
\forall x\in\X,r>0 ~\exists x_1,\ldots,x_\lambda\in\X:
B_r(x) \subseteq \cup_{i=1}^\lambda B_{r/2}(x_i)
}
}
.
$$
Then $\lambda$ is the {\em doubling constant} of $\X$,
and the {\em doubling dimension} of $\X$ is $\ddim(\X)=\log_2\lambda$. 

An {\em $r$-net} of a set $A\subseteq \X$ is any {\em maximal} subset
$A$ having mutual inter-point distance at least $r$.
The $r$-packing number $\M(r,A)$ of $A$ is the maximum 
size of any $r$-net of $A$:
\beqn
\label{eq:pack-def}
\M(r,A) = \max\set{ |E| : E\subseteq A, 
(x,y\in E)\wedge(x\neq y)
\implies\dist(x,y) \ge r}
.
\eeqn

\citet{GK-13} defined 
the {\em density constant} $\mu(\X)$ as
the smallest number such that any open $r$-radius ball in $\X$ 
contains at most $\mu$ points at mutual inter-point distance 
at least ${r}/{2}$:
\beqn
\label{eq:mudef}
\mu(\X) = \min\set{\mu\in\N:
(x\in\X)\wedge(r>0)\implies \M\paren{\frac r2,B_r(x)} \le \mu
},
\eeqn
and 
we define
the {\em density dimension} of $\X$
%is defined 
by
$\dens(\X)=\log_2\mu(\X)$.

\begin{sloppypar}
\paragraph{Learning model.}
We work in the standard {\em agnostic} learning model 
\citep{mohri-book2012,shwartz2014understanding},
whereby the learner receives a sample $S$ consisting of $n$ labeled examples
$(X_i,Y_i)$, drawn iid from an unknown distribution 
over $\X\times\set{-1,1}$. 
All subsequent probabilities and expectations will be with respect to 
this distribution.
Based on the training sample $S$,
the learner produces a {\em hypothesis} $h:\X\to\set{-1
%,0
,1}$,
whose {\em empirical error} is defined by $\serr(h)=n\inv\sum_{i=1}^n\pred{h(X_i)\neq Y_i}$
and whose {\em generalization error} is defined by $\err(h)=\P(h(X)\neq Y)$.
%(Thus, a prediction of $0$ always yields a mistake; see below.)
\hide{
The Bayes-optimal classifier, $h\bay$, is defined by 
$$h\bay(x)=\argmax_{y\in\set{-1,1}}\P(Y=y\gn X=x)$$
and 
$$R^*:=\err(h\bay)=\inf\set{\err(h)},$$ where the infimum is over all measurable hypotheses.
A learning algorithm mapping a sample $S$ of size $n$ to a hypothesis $h_n$ is said to be
strongly Bayes consistent if $\err(h_n)\ninf R^*$ almost surely.
}
\end{sloppypar}

\paragraph{Sub-sample, margin, and induced $1$-NN.} 
In a slight abuse of notation, we will blur the distinction
between $S\subset\X$ as a collection of points in a semimetric space
and $S\in (\X\times\set{-1,1})^n$ as a sequence
of labeled examples.
Thus, the notion of a {\em sub-sample} $\tilde S\subset S$
partitioned 
into its positively and negatively labeled subsets as
$\tilde S=\tilde S_+\cup \tilde S_-$ is well-defined.
The {\em margin} of $\tilde S$, defined by
\beqn
\label{eq:margdef}
\marg(\tilde S)
=
\dist(\tilde S_+,\tilde S_-),
\eeqn
is the minimum distance between a pair
of opposite-labeled points
(see Fig.~\ref{fig:marg} in the Appendix).
In degenerate cases where one of $\tilde S_+,\tilde S_-$ is empty, $\marg(\tilde S)=\infty$.
A sub-sample $\tilde S$ naturally induces the $1$-NN classifier
$h_{\tilde S}$, via
\beqn
\label{eq:1-nn-def}
h_{\tilde S}(x) = \sign(\dist(x,\tilde S_-)-\dist(x,\tilde S_+)).
\eeqn
%As an important technicality,
%we define $\sign(0)=0$, so that
%a test point whose
%equidistant nearest neighbors have conflicting labels is always
%misclassified. (In practice, one might flip a coin.)
%%Note that when $h_{\tilde S}(x) = 0$, $x$ has equidistant 
%%nearest neighbors with different labels, and so it is
%%misclassified.

\begin{sloppypar}
The problem of {\em nearest-neighbor condensing} is to produce the
minimal subsample $\tilde{S} \subset S$ so that the $1$-NN
classifier $h_{\tilde S}$ is {\em consistent} with $S$, i.e.\
has zero training error.
This problem was considered by \citet{DBLP:conf/nips/GottliebKN14} in the context
of doubling metric spaces, 
where
%and 
they demonstrated that it is
NP-hard to find the minimal $\tilde{S}$, 
even approximately 
(within a factor
$2^{O( (\ddim(S) \log(2\rad(S)/\marg(S)))^{1-o(1)})}$
of $|\tilde{S}|$). This result translates immediately to the more 
general semimetric spaces.
\end{sloppypar}

\hide{
\paragraph{Separability.} 
For $k\in\N$ and $\gamma>0$,
we say that the sample $S$ is
$(k,\gamma)$-{\em separable}
if there is a sub-sample
$S'\subset S$
such that
$|S\setminus S'|\le k$
and
$\marg(S')\ge\gamma$.
\hide{
\bit
\item[(i)]
and
\item[(ii)]
every pair of opposite-labeled points in $\tilde S$
is at least $\gamma$ apart in distance.
\eit
}
}

\section{Metric vs. Semimetric spaces}
\label{sec:semimetric}
In this section, we consider the basic tools used in learning algorithms
for doubling metric spaces. We show that 
in semimetric spaces, low doubling dimension does not
imply a low packing number (Lemma~\ref{lem:semi-pack}). 
Hence, all learning algorithms developed
for metric spaces relying on the doubling dimension 
are no longer efficient in semimetric spaces.
We then show that a low density constant does imply a low packing number,
even for semimetric spaces.
An even more stark distinction is established:
in doubling metric spaces,
the doubling and density
constants are never very far apart, while
in semimetric spaces the gap may be arbitrarily large.

These results 
suggest that the semimetric density constant will play the role of the metric doubling constant.
This intuition is borne out
in some aspects (Lemma~\ref{lem:metric-pack})
and proves to be spurious
in others (Lemma~\ref{lem:semi-nns}).
When controlling for both constants, 
approximate nearest-neighbor search
in semimetric spaces
cannot be performed nearly as efficiently as in doubling metric spaces.

The results presented in this section serve as the theoretical basis
motivating our learning algorithms (Section \ref{sec:algs}).

\subsection{Doubling constant vs. the density constant}

The following lemma states the well-known packing property of doubling spaces
(see for example \citet{KL04}).
It is a basic component of all the $\ddim$-based proximity methods.
Note the use of the triangle inequality
in the proof.

\begin{lemma}\label{lem:metric-pack}
If $\X$ is a metric space and $C\subseteq\X$ has
minimum inter-point distance $b$, then
%\begin{itemize}
%\item
$|C| \le \left( {2\rad(\X)}/{b} \right)^{O(\ddim(\X))}$.
%\end{itemize}
\end{lemma}

\begin{proof}
$C$ can be covered by $|C|$ open balls of radius $b$
centered at the points of $C$.
By repeatedly applying the definition of the doubling constant,
$C$ (and in fact all of $\X$) can be covered by 
$k 
= \lambda(\X)^{O(\rad(\X)/b)} 
= \left( \frac{2\rad(\X)}{b} \right)^{O(\ddim(\X))}$
balls of radius $\frac{b}{2}$ centered at points of $\X$.
By the triangle inequality, each of these 
$\frac{b}{2}$-radius balls
is completely contained in some $b$-radius ball
centered at points of $C$, hence $|C| \le k$.
\end{proof}

The central contribution of this section is the following lemma.
It demonstrates that for semimetrics, a doubling property does
not imply a packing property (unlike for metrics, Lemma~\ref{lem:metric-pack}).
However, a finite density constant does imply a packing property.

\begin{lemma}\label{lem:semi-pack}
In semimetric spaces, the doubling constant does not imply a
packing property, while the density constant does. In particular,
\begin{enumerate}[(a)]
\item
\label{it:const-lambda-large-C}
There exist semimetric spaces $\X$ of arbitrary
%size 
cardinality
with a universally bounded doubling constant
$\lambda(\X)=O(1)$, such that $\X$ contains a
$\rad(\X)$-net $C$ of size $\Theta(|\X|)$.
\item
\label{it:dens=>small-C}
For any semimetric space $\X$ and $b>0$, the size of 
any $b$-net of $\X$ is 
$$\left( \frac{2\rad(\X)}{b} \right)^{O(\dens(\X))}.$$
\end{enumerate}
\end{lemma}

\begin{proof}
%For the first claim: 
(\ref{it:const-lambda-large-C}).
Let $\X$ be composed of two sets, $A$ and $A'$.
Put $A=\set{a_1,\ldots,a_n}$,
endowed with
%$A$ consists of $n$ points, and 
%the distances in $A$ obey 
the line metric $\dist(a_i,a_j) = |i-j|$, 
so the maximum distance in $A$ is $n-1$.
Note that $\lambda(A) = O(1)$.
Define $A'$ to consist of $n$ points,
such that
%endowed with the 
%where 
%$\dist(a'_i,a_j) = \dist(a_i,a_j) = |i-j|$, 
\beq
\dist(a'_i,a_j) = \dist(a_i,a_j)+\phi\pred{i=j},
\qquad (\text{$\phi>0$ infinitesimal}),
\eeq
while 
$\dist(a'_i,a'_j) = n-1$.
This defines a semimetric on $\X$.

Clearly, $A'$ forms a $\rad(\X)$-net of size $|\X|/2$,
and yet we can show that $\lambda(\X) = O(1)$.
Indeed, consider any ball $B_r(x)$ in $\X$. Then 
all points in $B_r(x)$ can be covered by the same $\lambda(A)=O(1)$ 
balls of radius $\frac{r}{2}$
that cover $A \cap B_r(x)$.
The claim follows.

(\ref{it:dens=>small-C}).
%For the second claim: 
%Let $\X$ have 
%radius $R$. 
Suppose the radius of $\X$ is $R$.
Partition $\X$ into clusters by extracting
from $\X$ an arbitrary net $D$ with minimum inter-point distance $R/2$,
and assigning each point $p \in \X$ to a cluster centered at 
the nearest neighbor of $p$ in $D$. Then apply the procedure recursively to
each cluster (halving the previous radius), 
until reaching point sets
with minimum inter-point distance at least $b$. 
Clearly, an appropriate choice of the subsets can yield a final set containing $C$.
For example, the first set may contain all points in the $R/2$-net of $C$,
the second all points in the $R/4$-net of $C$, etc.
By repeatedly applying the definition
of the density constant, the size of the final set is bounded by 
$\mu(\X)^{\log_2 (2\rad(\X)/b) }
= \left( \frac{2\rad(\X)}{b} \right)^{O(\dens(\X))}$,
and this bounds $|C|$ as well.
\end{proof}

In fact, a deeper principle underlies the results above:
In metric spaces, the doubling and density constants are almost the same,
while in semimetric spaces there may be a large gap between them. 
This is captured in the following lemma, which delineates the relationship
between the doubling constant and density constant.
(The first half of the lemma is due to \citet{GK-13}.)

\begin{lemma}\label{lem:metric-semimetric}
Let $\X$ be point set endowed with a metric distance function. Then
\begin{enumerate}[(a)]
\item
\label{it:lamX-muX}
$\lambda(\X) \le \mu(\X)$,
\item 
\label{it:sqrt-mu-lam}
$\sqrt{\mu(\X)} \le \lambda(\X)$.
\newcounter{enumi_saved}
\setcounter{enumi_saved}{\value{enumi}}
\end{enumerate}
Let $\Y$ be a point set endowed with a semimetric distance function. Then
\begin{enumerate}[(a)]
\setcounter{enumi}{\value{enumi_saved}}
\item
\label{it:lamY-muY}
$\lambda(\Y) \le \mu(\Y)$,
\item 
\label{it:my-lam-gap}
$\mu(\Y)$ may be as large as $\Theta(|\Y|)$, even when $\lambda(\Y) = O(1)$.
\end{enumerate}
\end{lemma}

\begin{proof}
To prove 
(\ref{it:lamX-muX}) and (\ref{it:lamY-muY}),
%the first and third items, 
that $\lambda \le \mu$:
Consider any open ball $B_r(x) \in \X$. 
Let $C$ be a maximal collection of points at mutual inter-point distance at least
$\frac{r}{2}$, and note that by definition
$|C| \le \mu(\X)$. 
By the maximality of $C$, $|C|$ balls of radius $\frac{r}{2}$ 
centered at points of $C$ cover all of $B_r(x)$, so 
$\lambda(\X) \le |C| \le \mu(\X)$.
For (\ref{it:sqrt-mu-lam}):
%To prove the second item, 
%that $\sqrt{\mu(\X)} \le \lambda(\X)$:
again, consider any open ball $B_r(x) \in \X$,
and let $C$ be a maximal collection of points at mutual 
inter-point distance at least $\frac{r}{2}$.
Now, by definition $\X$ may be covered by $\lambda(\X)$ balls
of radius $\frac{r}{2}$, and each of these smaller balls
may be covered by $\lambda(\X)$ balls of radius $\frac{r}{4}$,
so there exists a set of $\lambda^2(\X)$ balls of radius 
$\frac{r}{4}$ covering all of $X$, and in particular $C$.
By the triangle inequality, each ball of radius $\frac{r}{4}$
can cover at most one point of $C$, and so 
$|C| \le \lambda^2(\X)$.
Finally, (\ref{it:my-lam-gap}) 
%The fourth item 
follows immediately from 
%the two items of 
Lemma~\ref{lem:semi-pack}.
\end{proof}

\section{Basic constructions and the density constant}
Before presenting our classification algorithms in 
Section \ref{sec:algs},
we will show how to execute
two basic constructions --- $r$-net
and nearest neighbor search --- 
for semimetrics with finite density constant. These results
are strictly
%than what is possible for metric spaces.
worse than the corresponding state of the art for metric spaces.

\begin{sloppypar}
\paragraph{Net extraction and condensing.}
In Lemma~\ref{lem:semi-pack} above, we bounded the $r$-packing number 
of semimetric spaces, which in turn bounds the size of the largest
$r$-net of the space.
For a metric set $S$, 
it is known how to extract an $r$-net in time 
$2^{O(\ddim(S))} |S| \min \{\log (\rad(S)/r), \log |S|\}$ \citep{KL04,HM06,CG06}.
%We can show the following 
The following result holds
for semimetric spaces.
\end{sloppypar}

\begin{lemma}\label{lem:semi-net}
Given a set $S$ equipped with a semimetric distance function, an 
$r$-net of $S$ of size 
$$k = \mu(S)^{ \log_2(2\rad(S)/b) }
= \left( \frac{2\rad(S)}{b} \right)^{O(\dens(S))}$$
can be extracted in time 
$O(k|S|)$.
\end{lemma}

\begin{proof}
We greedily build an $r$-net for $S$.
Initialize set $C=\emptyset$, and for every point in 
$S$, add it to $C$ if its closest neighbor in $C$ is at distance $r$ or greater.
By Lemma~\ref{lem:semi-pack}, $|C| \le k$, and so the total runtime is $O(k|S|)$.
See Algorithm~\ref{alg:bf} in the Appendix.
\end{proof}

%\bigskip

\paragraph{Nearest neighbor search.}
Finally, we juxtapose the time bounds for nearest neighbor search in metric
and semimetric spaces.
In metric spaces, the following bounds on exact and approximate nearest 
neighbor search are well-known (the proof is deferred to the Appendix):

\begin{lemma}\label{lem:metric-nns}
Given a point set $S$ equipped with a metric distance function, and a query point $x$:
\begin{enumerate}[(a)]
\item
\label{it:metric-nns-1}
Locating the exact nearest neighbor of $x$ in $S$ requires
$\Theta(|S|)$ comparisons in the worst case.
\item
\label{it:metric-nns-2}
A $(1+\eps)$-approximate nearest neighbor of $x$ in $S$ can be found in time 
$$2^{\ddim(S)}\log |S| + \eps^{-O(\ddim(S))}.$$
\end{enumerate}
\end{lemma}

For semimetric spaces, we demonstrate that the situation is much worse:

\begin{lemma}\label{lem:semi-nns}
Given a point set $S$ equipped with a semimetric distance function, 
discovering an exact or approximate nearest neighbor requires 
$\Theta(|S|)$ comparisons in the worst case.
\end{lemma}

\begin{proof}
For the upper bound, trivially $O(|S|)$ time is sufficient to consider every point in $S$.

For the lower bound, suppose the query point $q$ is at an infinitesimally small
distance from a single point $s_0 \in S$, and at distance $2\rad(S)$ from
all other points of $S$. Then $s_0$ can be any point in $S$, and cannot 
be located without inspecting each point:
Without the triangle inequality,
the distance between one pair of points has no bearing on any other distance.
\end{proof}

\section{Classification algorithms}
\label{sec:algs}
In this section, we present a classification algorithm for 
semimetric spaces. For a labeled sample $S$, 
recall that the {\em margin} of $S$ is
the minimum distance between oppositely labelled points in $S$,
as defined formally in (\ref{eq:margdef}).
The margin of a given sample
%If the margin of $S$ is not known, it
can be computed in time 
$\Theta(|S|^2)$ by considering all pairs of points.

We consider the problems of producing both consistent and
inconsistent $1$-NN classifiers for the sample 
(see Section~\ref{sec:prelim}).
We begin with a consistent classifier.

\begin{theorem}\label{thm:consistent}
Let $S$ be a sample set equipped with a semimetric distance function,
and let the margin $\gamma$ of $S$ be given. In time 
$O(k|S|)$
we can construct a nearest-neighbor classifier that 
achieves zero training
error on $S$, 
where 
$k = 
\left(
\frac{2\rad(S)}{\gamma}
\right)^{O(\dens(S))}
.$
The evaluation time for a test point is $O(k)$, and 
with probability $1-\delta$,
the resulting classifier has generalization error
%the probability 
%of misclassification is
$
O 
\left(
\frac{k \log n + \log \frac{1}{\delta}}{n}
\right)
$.
\end{theorem}

\begin{proof}
We build a $\gamma$-net $C$ for $S$ in time $O(k|S|)$, as in Lemma~\ref{lem:semi-net}.
Since $\gamma$ is the margin, by construction every point in $S$ has the same label as its
nearest neighbor in $C$, and so the nearest neighbor classifier with 
regards to $C$ has zero sample error.

Given a test point $x$, we assign it the same label as its nearest neighbor
in $C$. By Lemma~\ref{lem:semi-nns}, $\Theta(k)$ operations are necessary 
and sufficient to locate the nearest neighbor. The generalization bounds follow
from Theorem \ref{thm:gen-slow}(i).
\end{proof}

\begin{sloppypar}
The procedure in
%Note that 
Theorem~\ref{thm:consistent} compresses $S$, 
producing a consistent sub-sample $C$.
Immediate from the theorem is that 
the smaller the compressed set $C$, the
better the generalization bounds of the classifier.
However, as \citet{DBLP:conf/nips/GottliebKN14} recently demonstrated,
even in metric spaces, it is NP-hard to 
approximate the size of the minimum consistent subset to within a factor
$2^{O( (\ddim(S) \log(2\rad(S)/\marg(S)))^{1-o(1)}}
= 2^{O( (\dens(S) \log(2\rad(S)/\marg(S)))^{1-o(1)}}$
(where the equality follows from Lemma~\ref{lem:metric-semimetric}).
This means that choosing the net of Lemma~\ref{lem:semi-net} 
is close to the optimal construction for a consistent subset of $S$.
\end{sloppypar}

It is natural to ask 
whether allowing the classifier nonzero sample error
results in improved generalization bounds.
This is indeed generally the case, as the bound in 
Theorem~\ref{thm:gen-fast} indicates.
Optimizing this bound is an instance of Structural Risk Minimization (SRM).
Unfortunately, we can show SRM to be infeasible for this problem:

\begin{theorem}\label{thm:hard}
Given a set $S$ equipped with a metric or semimetric distance function, let
$S^* \subset S$ be a sub-sample
%witness 
for which the generalization
bound $Q(d,\eps)$ in Theorem~\ref{thm:gen-fast} (for a fixed constant $\delta$) 
is minimized. Then it is NP-hard to compute any subset of $S$ achieving a
generalization bound within factor 
$2^{O( (\dens(S) \log(2\rad(S)/\marg(S)))^{1-o(1)}}$
of the generalization bound induced by $S^*$.
\end{theorem}

\begin{proof}
The proof is via reduction from the minimum consistent subset 
problem, which was shown 
by~\citet{DBLP:conf/nips/GottliebKN14}
to be hard to approximate. Fix the confidence level $\delta$ in the bound, 
let $T$ be an instance of the minimum consistent subset
problem, and put $m=|T|$. For some large value $p$, 
replace each point $t_i \in T$ with a set
of $p$ points $s_{i,1},\ldots,s_{i,p}$ obeying the line metric,
so that $\dist(s_{i,a},s_{i,b}) = \phi|a-b|$ for an infinitesimally
small $\phi$. Put $\dist(s_{i,a},s_{j,b}) = \dist(t_i,t_j)$.
The new set is $S$, with $n=|S|=pm$.

Consider a subset $S' \subset S$. If the $1$-NN rule on $S'$
misclassifies a point of $S$, say $s_{i,a}$, then in fact it 
must misclassify all $p$ points $s_{i,b}$, $b \in [1,p]$. So an
inconsistent subset of $S$ achieves a value of
$Q(|S'|,p/n)=\Omega(p/n)$ in the generalization bound.

Now consider the consistent subset of $S$ consisting of $m = n/p$ 
points $s_{i,1}$ for $i \in [1,m]$.
This classifier achieves a generalization bound of
$O \left( \frac{m \log n}{n} \right)
= O \left( \frac{\log n}{p} \right)$. 
So when
$p = \Omega(\sqrt{n \log n})$, 
this consistent classifier is better than any inconsistent classifier,
and by increasing $p$ we can amplify this gap arbitrarily.
Now a consistent subset of size
$d \le m$ has generalization bound 
$O \left( \frac{d\log n}{n} \right)$.
As it is NP-hard to find a subset whose size is within a factor
$2^{O( (\dens(S) \log(2\rad(S)/\marg(S)))^{1-o(1)}}$ 
of the smallest consistent subset, it is NP-hard to find a 
consistent subset with generalization bound within a factor
$2^{O( (\dens(S) \log(2\rad(S)/\marg(S)))^{1-o(1)}}$ 
of the optimal consistent subset, and the theorem follows.
\end{proof}

Let us turn our attention to the margin-based
%Instead, we will consider the 
generalization bound provided by 
Theorem~\ref{thm:gen-sep}.
As before, we wish to perform SRM for this bound.
%find a classifier attaining the minimum
%value of the generalization bound --- that is, achieving the
%optimal tradeoff between compression and loss. 
Fortunately, we are able to compute the latter exactly in polynomial time,
and even more efficiently if we are willing to settle for
a solution within a constant factor of the optimal:

\begin{theorem}
\label{thm:main-alg}
Given a sample set $S$ equipped with a semimetric:
\begin{enumerate}[(a)]
\item
A nearest-neighbor classifier minimizing the generalization bound of
Theorem~\ref{thm:gen-sep} 
can be computed in randomized time $O(|S|^{4.373})$. 
\item
\label{it:improved-alg}
A nearest-neighbor classifier whose generalization bound is within
factor 2 of optimal can be computed in deterministic time 
$O(|S|^2 \log |S|)$.
\end{enumerate}
Each of these classifiers can be evaluated on test points in time
$\left(
\frac{\rad(S)}{\gamma}
\right)^{O(\dens(S))}
$,
where $\gamma$ is the margin imposed by the SRM procedure.
\end{theorem}

\begin{proof}
For each of these solutions, we enumerate and sort in increasing
order the distances between all oppositely labelled point pairs in $S$, 
in total time 
$O(|S|^2 \log |S|)$. 
Each distance constitutes a separate guess for the optimal 
margin to ``impose'' on $S$. That is, for each distance $\gamma$, we
will remove from $S$ some points to ensure that all opposite
labelled pairs are more than $\gamma$ far apart.
%farther than this distance.

To accomplish this, we iteratively build a new graph $G$.
We initialize $G$ 
with vertices representing the points of $S$.
At each round we add to $G$ an edge between the next closest
pair of opposite labelled points, as given by the sorted enumeration above.
This distance is the margin of the current round:
Points connected by an edge in $G$ represent pairs that are too 
close together for the current margin, and we need to compute
how many points must be removed from $G$ in order for no edge to 
remain in the graph. (However, no points or edges will actually
removed from $G$.)
As observed by \citet{DBLP:journals/tit/GottliebKK14}, this task is precisely the problem
of bipartite vertex cover.
By K\"onig's theorem, the minimum vertex cover problem in bipartite graphs
is equivalent to the maximum matching problem,
and a maximum matching in bipartite graphs can be computed in randomized time
$O(n^{2.373})$ \citep{1033180, W12}. So for each candidate margin, we
can compute in $O(n^{2.373})$ time the number of points that
must be removed from the current graph $G$ in order to remove all
edges. For $O(n^2)$ possible margins, this amounts to
$O(n^{4.373})$ time. Having computed for each
inter-point distance the number of points required to be deleted
to achieve this distance, we choose the distance-number pair which
minimizes the bound of 
Theorem~\ref{thm:gen-sep}. 
We then remove these points from
$S$, and use the algorithm of
Lemma~\ref{lem:semi-net} to construct a net satisfying the margin bound.

\begin{sloppypar}
The runtime improvement in 
(\ref{it:improved-alg})
%the second item 
comes from a faster vertex-cover computation.
It is well known that a $2$-approximation to vertex cover can be computed
(in arbitrary graphs) by a greedy algorithm in time linear in the graph size
$O(|V^+ \cup V^-|+|E|) = O(n^2)$, see e.g.~\citet{DBLP:journals/jal/Bar-YehudaE81}.
This algorithm simply chooses any edge and removes both endpoints, until
no edges remain. We apply this algorithm to our setting:
Copy set $S$ to $T$, and iteratively remove from $T$ the next closest
pair of oppositely-labelled points, as given by the sorted enumeration above.
For each distance, we record how many points have been removed from $T$, and
this is a 2-approximation for the minimum number of points that must be
removed in order to attain this distance. 
 Having computed for each
inter-point distance the number of points required to be deleted
to achieve this distance, we choose the distance-number pair which
minimizes the bound of 
Theorem~\ref{thm:gen-sep}. We then remove these points from
$S$, and use the algorithm of
Lemma~\ref{lem:semi-net} to construct a net satisfying the margin bound.
The runtime is dominated by the time required to sort the distances.
\end{sloppypar}

For both algorithms, a new point is classified by finding its nearest neighbor in the extracted net.
\end{proof}

\section{Generalization guarantees}
\label{sec:learn}

In this section, we provide general sample compression
bounds, which then will be specialized to 
the
nearest-neighbor classifier
proposed above.
Theorem~\ref{thm:gen-fast}
presents a smooth interpolation between
two classic bounds: the consistent case with rate
$\tilde O(1/n)$, and the agnostic case with rate
$\tilde O(1/\sqrt n)$.
Applied to margin-based semimetric sample-compression schemes, 
this result yields the efficiently computable and optimizable
bound in Theorem~\ref{thm:gen-sep}, which is nearly optimal
(as shown in Theorem~\ref{thm:hard}).
Finally,
the lower bound in Theorem~\ref{thm:sample-lb}
shows that 
even under margin assumptions, there exist adversarial distributions
forcing
the sample complexity to
be exponential in $\dens$.

\subsection{Sample compression schemes}
We use the notion of a {\em sample compression scheme}
in the sense of \citet{DBLP:journals/ml/GraepelHS05},
where it is treated in full rigor.
Informally, a learning algorithm 
maps a sample $S$ of size $n$ to a hypothesis $h_S$.
It is a {$d$-sample compression scheme} if a sub-sample of size $d$
suffices to produce a hypothesis that agrees with the labels
of all the $n$
points.
It is an 
{\em $\eps$-lossy}
$d$-sample compression scheme if a sub-sample of size $d$
suffices to produce a hypothesis that
disagrees with the labels of
at most $\eps n$ 
of
the $n$ sample points.

The algorithm 
need not know $d$ and $\eps$ in advance.
We say that 
the sample $S$
is {\em $(d,\eps)$-compressible} 
if 
the algorithm succeeds in finding
an
{$\eps$-lossy}
$d$-sample compression scheme for this particular sample.
In this case:

\hide{
We say that $\tilde S\subset S$ is $\eps$-{\em consistent} with $S$ if
$
\oo{n}\sum_{x\in S}\pred{ \nu_S(x)\neq \nu_{\tilde S}(x)} \le \eps.
$
For $\eps=0$, an $\eps$-consistent $\tilde S$ is simply said to be {\em consistent}
(which matches our previous notion of consistent subsets).
A sample $S$ is said to be $(\eps,\gamma)$-{\em separable} (with witness $\tilde S$) 
if there is an $\eps$-consistent $\tilde S\subset S$
with $\gamma(\tilde S)\ge\gamma$.

We begin by invoking a standard Occam-type
argument to show that the existence of small $\eps$-consistent sets
implies good generalization.
}

\begin{theorem}[\citet{DBLP:journals/ml/GraepelHS05}]
\label{thm:gen-slow}
For any distribution over $\X\times\set{-1,1}$,
any $n\in\N$ and any $0<\delta<1$,
with probability at least $1-\delta$ over the random sample 
$S
$ of size $n$, the following holds:
\begin{enumerate}
\item[(i) ]If $S$ is 
{$(d,0)$-compressible}, then
${\ds
\err(h_S) \le \oo{n-d}\paren{
(d+1)\log n
+\log\oo\delta}.
}$
\item[(ii)] 
If $S$ is 
{$(d,\eps)$-compressible}, then
${\ds
\err(h_S) \le
\frac{\eps n}{n-d}+\sqrt{
\frac{
(d+2)\log n
+\log\oo\delta}{2(n-d)}
}.
}$
\end{enumerate}
\end{theorem}
\begin{sloppypar}
The generalizing power of sample compression was independently 
discovered by \citet{warmuth86,MR1383093}, and later elaborated upon by 
\citet{DBLP:journals/ml/GraepelHS05}.
The bounds above are already quite usable, but they feature an abrupt transition from the
$(\log n)/n$ decay in the
lossless ($\eps=0$) regime to the 
$\sqrt{(\log n)/n}$ decay in the
lossy regime. We now provide a smooth interpolation between the two
(such results are known in the literature as ``fast rates'' \citep{CambridgeJournals:8212764}):
\end{sloppypar}

\begin{theorem}
\label{thm:gen-fast}
Fix a distribution over $\X\times\set{-1,1}$,
an $n\in\N$ and 
$0<\delta<1$.
With probability at least $1-\delta$ over the random sample 
$S
$ of size $n$, the following holds for all $0\le\eps\le\oo2$:
If
$S$ is 
{$(d,\eps)$-compressible}, then
\beqn
\label{eq:qdef}
\err(h_S) \le
\tips
+
\frac{2}{3(n-d)}\log\frac{n^{d+2}}\delta
+
\sqrt{ \frac{9\tips(1-\tips)}{2(n-d)}\log\frac{n^{d+2}}\delta}
=: Q(d,\eps)
,
\eeqn
where
$
\tips = \frac{\eps n}{n-d}.
$
\end{theorem}
\begin{proof}
We closely follow the argument in
\citet[Theorem 2]{DBLP:journals/ml/GraepelHS05},
with the twist that instead of Hoeffding's inequality,
we use Bernstein's.
The particular form of the latter is due to
\citet[Lemma 1]{DBLP:conf/icml/DasguptaH08}:
if $\hat p\sim\bin(n,p)/n$ and $\delta>0$,
then
\beqn
\label{eq:emp-bern}
p \le \hat p + 
\frac{2}{3n}\log\frac{1}\delta
+
\sqrt{ \frac{9\hat p(1-\hat p)}{2n}\log\frac{1}\delta}
\eeqn
holds with probability at least $1-\delta$.

Now suppose that $S$ is $(d,k/n)$-compressible,
as witnessed by some sub-sample $\tilde S\subset S$ of size $d$.
In particular, the hypothesis $h_{\tilde S}
$
induced by the sub-sample $\tilde S$
makes $k$  
or fewer
mistakes on the $n-d$ points in $S\setminus\tilde S$.
Substituting $p=\err(h_{\tilde S})$ and
\beq
\hat p = \serr_{S\setminus\tilde S}(h_{\tilde S})
%&:=&
:=
\oo{|{S\setminus\tilde S}|}
\sum_{x\in S\setminus\tilde S} \pred{\text{$h_{\tilde S}$ makes a mistake on $x$}}
%\\
%&\le&
\le
\frac{k}{n-d}=\tips
\eeq
into (\ref{eq:emp-bern})
yields that 
%conditioned on $\tilde S$,
for fixed $\tilde S$ and random $S\setminus\tilde S$,
with probability at least $1-\delta$,
\beqn
%\nonumber
\err(h_{\tilde S})
&\le &
\serr_{S\setminus\tilde S}(h_{\tilde S})
+ 
\frac{2}{3(n-d)}\log\frac{1}\delta
%\\
\label{eq:err-serr}
%&+&
+
\sqrt{ \frac{9\tips(1-\tips)}{2(n-d)}\log\frac{1}\delta},
\eeqn
where we used the monotonicity of $t\mapsto t(1-t)$ on $[0,\oo2]$.
To see that (\ref{eq:err-serr}) follows from (\ref{eq:emp-bern}),
note that
when $\tilde S$ of size $d$ is fixed and $S\setminus\tilde S$
is drawn iid $\sim \P$,
we have $(n-d)\serr_{S\setminus\tilde S}(h_{\tilde S})\sim\bin(n-d,\err(h_{\tilde S})).$
To make (\ref{eq:err-serr}) hold simultaneously for all $\tilde S\subseteq S$,
divide $\delta$ by $n^d$ --- the number of ways to choose a (multi)set $\tilde S$ of size $d$.
To make the claim hold for all $d\in[n]$ and all $0\le\eps<1$,
stratify (as in \citet[Lemma 1]{DBLP:journals/ml/GraepelHS05})
over the $n^2$ possible choices of $d$ and $k$, which amounts to dividing $\delta$
by an additional factor of $n^2$.

\end{proof}

\subsection{Margin-based nearest neighbor compression}

We now specialize the general sample compression result
of Theorem~\ref{thm:gen-fast} to our setting,
where $h_{S'}$ induced by a sub-sample $S'\subset S$ is given
by the $1$-NN classifier defined in (\ref{eq:1-nn-def}).
Any sample $S$ of size $n$ is trivially
$(n,0)$-compressible and $(0,\oo2)$-compressible
--- the former is achieved by not compressing at all,
and the latter by a constant predictor.
Now $d$ and $\eps$ cannot simultaneously be made arbitrarily small,
and for non-degenerate samples $S$,
the bound $Q$ in Theorem~\ref{thm:gen-fast}
will have a nontrivial minimal value $Q^*$.
%we may define, for a fixed $S$,
%\beq
%\eps^*(d,S) = \argmin\set{ Q(d,\eps) : \eps>0,\text{ $S$ is $(d,\eps)$-compressible}},
%\eeq
%where 
Theorem~\ref{thm:hard} shows that computing $Q^*$ is intractable
and the
%Computing or even approximating $\eps^*(d,S)$ seems to be intractable
%in light of the hardness of approximation results in Theorem~\ref{thm:hard}.
%The 
algorithm in Theorem~\ref{thm:main-alg} solves a tractable modification
of this problem. 
For $k\in\N$ and $\gamma>0$,
let us say that the sample $S$ is
$(k,\gamma)$-{\em separable}
if it admits a sub-sample
$S'\subset S$
such that
$|S\setminus S'|\le k$
and
$\marg(S')>\gamma$,
and observe that separability implies compressibility:
\begin{lemma}
If $S$ is $(k,\gamma)$-separable then it is 
$\paren{\mu(S)^{{\log_2(2\rad(S)/\gamma)}},
%\eps\le
\frac{k}{|S|}}$-compressible.
\end{lemma}
\bepf
Suppose $S'\subset S$ is a witness of 
$(k,\gamma)$-separability.
Being pessimistic, we will allow our lossy sample compression scheme
to mislabel all of $S\setminus S'$, but not any of $S'$, giving it a sample error
$\eps \le \frac{k}{|S|}$. Now by construction,
$S'$ is $(0,\gamma)$-separable, and thus a $\gamma$-net $\tilde S\subset S'$
suffices to recover the correct labels of $S'$ via $1$-nearest neighbor.
Lemma~\ref{lem:semi-pack} provides the estimate
$
|\tilde S| \le \mu(S)^{{\log_2(2\rad(S)/\gamma)}},
$
whence the compression bound.
\enpf

These observations culminate in an efficiently optimizable
margin-based generalization bound:

\begin{theorem}
\label{thm:gen-sep}
Fix a distribution over $\X$,
an $n\in\N$ and $0<\delta<1$.
With probability at least $1-\delta$ over the random sample 
$S
$ of size $n$, the following holds
for all $0\le k\le n/2$:
If $S$ is 
{$(k,\gamma)$-separable} with witness $S'$, then
%\beq
$
\err(h_{S'}) \le
%\tips
%+
%\frac{2}{3(n-d)}\log\frac{n^{d+2}}\delta
%+
%\sqrt{ \frac{9\tips(1-\tips)}{2(n-d)}\log\frac{n^{d+2}}\delta}
Q(d,k/n)
=: R(k,\gamma)
,
$
%\eeq
where
$Q$ is defined in (\ref{eq:qdef})
and
%$
%\tips = \frac{k}{n-d}
%$
%and
$d = \mu(S')^{{\log_2(2\rad(S')/\gamma)}}$.
Furthermore, the minimizer $(k^*,\gamma^*)$ of $R(\cdot,\cdot)$
is efficiently computable.
\end{theorem}

\subsection{Sample complexity lower bound}
The following result shows
that 
even under margin assumptions, 
a sample of size exponential in $\dens$
will be required for some distributions.

\begin{theorem}
\label{thm:sample-lb}
%There exists a semimetric space $(\X,\dist)$ 
%and a distribution $\P$ over
%$\X$ 
For every semimetric space $(\X,\dist)$, there is
a distribution $\P$
such that 
$\err(f)=0$
for some ``target'' concept $f:\X\to\set{-1,1}$,
yet
for any learning algorithm mapping samples $S$ of size $n$
to hypotheses $h_n:\X\to\set{-1,1}$, we have,
with high probability,
%\beq
$
\err(h_n) = \Omega\paren{\frac{ \sqrt{\mu(\X)^{{\log_2(2\rad(S)/\marg(S))}}}}{n}} 
.
$
%\eeq
%with high probability.
\end{theorem}

\begin{proof}
The definition of the density constant implies the existence
of $k=\mu(\X)=2^{\dens(\X)}$ nearly equidistant points $\set{x_i}$,
such that $1\le\dist(x_i,x_j)\le2$ for all $1\le i<j\le k$.
%We take $\X$ to be a set of $k$ equidistant points: $\dist(x,x')=1$ for all $x\neq x'$ in $\X$.
Following the standard VC lower bound argument 
\citep{MR1072253,Ehrenfeucht1989247},
we construct $\P$ by putting a mass of $1-8\eps$ on one of the $k$ points
and distributing the remaining mass uniformly over the other $k-1$ points.
The target $f:\set{x_i}\to\set{-1,1}$ is drawn uniformly at random
from among the $2^k$ choices,
so as to thwart any
learning algorithm.\hide{
Take $\F$ to be the collection of all 
%non-constant 
concepts $g:\X\to\set{-1,1}$.
Then $|\F|=2^k$ and 
the VC-dimension of $\F$ is $k$.
The teacher will choose the target $f\in\F$ uniformly at random 
so as to thwart any
learning algorithm 
---
and in particular, this choice ensures that
indeed
%(\ref{eq:nondeg-marg}) 
$\marg(S)<\infty$
with high probability.
%(including, in particular, our sample compression scheme).
}
For fixed $0<\eps<\oo8$
and
$0<\delta<\oo{100}$,
this choice 
ensures
that 
a sample of size
$
\Omega\paren{\frac{k}{\eps}}
$
is required in order to produce an $\eps$-accurate hypothesis with $\delta$-confidence.
Inverting for $\eps=\err(h_n)$ will yield the claim
--- as soon as 
%a relationship
%between 
$k$ and 
%$\mu(\X)$, 
%$\rad(S)$, $\marg(S)$
$\ell:=\mu(\X)^{{\log_2(2\rad(S)/\marg(S))}}$
can be tied together.

By construction, $0<\marg(S)\le\rad(S)<\infty$,
except for two possible degenerate cases:
(a) $\rad(S)=0$ and (b) $\marg(S)=\infty$.
Case (a) occurs when
$S$ consists of a single point,
with probability decaying as $e^{-8\eps n}$.
Case (b) occurs when $f$ assigns the same label to all $k$
points, with probability $2^{-k+1}$.
Thus, with overwhelming probability,
$\log_2(2\rad(S)/\marg(S))\ge1$.
Since $\rad(S)\le2$, by construction,
we also have
$\log_2(2\rad(S)/\marg(S))\le2$. 
It follows that $k\le\ell\le k^2$,
which yields the claim.
\end{proof}

\section*{Acknowledgements}
We thank Daniel Hsu for communicating to us the version of Bernstein's inequality
appearing in (\ref{eq:emp-bern}). We thank Roi Weiss for helpful comments on the manuscript.

\bibliographystyle{plainnat} 
\bibliography{../../mybib}

\newpage

\appendix
\section{Figures and deferred proofs}

\subsection*{Figure accompanying the definition:
Sub-sample, margin, and induced $1$-NN.}

\tikzset{
    plusS/.style = {  },
    plustS/.style = {double = black!30, double distance=1pt},
    minusS/.style = {},
    minustS/.style = {double = black!30, double distance=1pt}
}
\begin{figure}[H]
\centering
\begin{tikzpicture}[scale = 1,
every node/.style = {draw, shape = circle, minimum size = 2.1mm,inner sep = 0mm, fill = black!10}
]
\node  [plusS] (p1) at (-1.1,-.2) {$+$};
\node  [minusS] (m1) at (-1,-1) {$-$};
\node  [plustS] (p2) at (.0, .5) {$+$};
\node  [minustS] (m2) at (-.3,-1.5) {$-$};
\node  [plustS] (ps) at (-1.7,.6) {$+$};
\node  [plustS] (ps3) at (1.1,.8) {$+$};
\node  [minustS] (mts) at (1.3,-1.4) {$-$};
\node  [minusS] (ms) at (.7,0) {$-$};
\node  [minustS] (ms) at (-2,-1.5) {$-$};
\draw  [dashed] (m1) -- (p1) node [left=7pt, midway, draw=none, fill = none] {$\marg(S)$};
\draw  [dashed](p2) -- (m2) node [above right=18pt,draw=none, fill = none] {$\marg(\tilde S)$};
\end{tikzpicture}
\caption{In this example, the sub-sample $\tilde S\subset S$ is indicated by double circles.
It is always the case that
$\marg(\tilde S)\ge\marg(S)$.}%
\label{fig:marg}%
\end{figure}
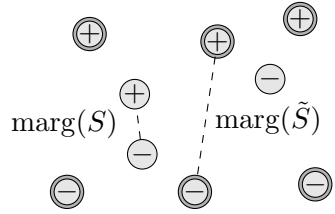

\subsection*{Algorithm accompanying Lemma~\ref{lem:semi-net}}

\begin{algorithm}[H]
   \caption{Brute-force net construction}
   \label{alg:bf}
\begin{algorithmic}
 \REQUIRE sample $S$, margin $r$
 \ENSURE $C$ is an $r$-net for $S$
 \FOR{$x \in S$}
    \IF{$\dist(x,C) \ge r$}
        \STATE $C = C \cup \{x\}$
    \ENDIF
 \ENDFOR
\end{algorithmic}
\end{algorithm}

\subsection*{Proof of Lemma~\ref{lem:metric-nns}}

\begin{proof}
To prove 
(\ref{it:metric-nns-1}),
%the first claim, 
let $S$ be a set of points obeying the
line metric, i.e.\ the distance between $s_i,s_j \in S$ is
$|i-j|$. Suppose $x$ is at distance $n=|S|$ from $s_i$, and at distance
$n+1$ from all other points of $S$. Then $s_i$ can be any point of $S$, 
and cannot be located without inspecting each point.
The 
%second 
claim 
in
(\ref{it:metric-nns-2})
is the result of \citet{KL04}.
\end{proof}

\end{document}